\documentclass{article}
\PassOptionsToPackage{square, sort, numbers}{natbib}

\usepackage[final]{neurips_2022}

\usepackage[utf8]{inputenc} 
\usepackage[T1]{fontenc}    
\usepackage{hyperref}       
\usepackage{url}            
\usepackage{booktabs}       
\usepackage{amsfonts}       
\usepackage{nicefrac}       
\usepackage{microtype}      
\usepackage{xcolor}         

\usepackage{microtype}
\usepackage{graphicx}
\usepackage{subcaption}
\usepackage{booktabs} 
\usepackage{multirow}
\usepackage{amsmath}
\usepackage{amssymb}
\usepackage{wrapfig}

\usepackage{hyperref}

\usepackage{cleveref}
\usepackage{mathtools}
\usepackage{xparse}
\usepackage{xcolor}
\usepackage{amsthm}
\usepackage{amsmath}

\newtheorem{definition}{Definition}
\newtheorem{theorem}{Theorem}
\newtheorem{lemma}{Lemma}
\theoremstyle{definition}

\DeclarePairedDelimiter\floor{\lfloor}{\rfloor}

\newcommand{\xv}{\mathbf{x}}
\newcommand{\zv}{\mathbf{z}}

\newcommand{\sbpm}[1]{\scalebox{0.7}{$\pm$#1}}
\NewDocumentCommand{\ming}{ mO{} }{\textcolor{red}{\textsuperscript{\textit{Ming}}\textsf{\textbf{\small[#1]}}}}

\title{Few-Shot Non-Parametric Learning with Deep Latent Variable Model}

\author{%
  Zhiying Jiang$^{1,2}$ \quad Yiqin Dai$^{2}$ \quad Ji Xin$^{1}$ \quad Ming Li$^{1}$ \quad Jimmy Lin$^{1}$ \\
  $^1$ University of Waterloo \quad $^2$ AFAIK. \\
  \texttt{zhiying.jiang@uwaterloo.ca} \quad \texttt{phinodadai@gmail.com} \\
  \texttt{\{ji.xin, mli, jimmylin\}@uwaterloo.ca} \\
}

\begin{document}

\maketitle

\begin{abstract}

Most real-world problems that machine learning algorithms are expected to solve face the situation with 1) unknown data distribution; 2) little domain-specific knowledge; and 3) datasets with limited annotation. 
We propose Non-Parametric learning by Compression with Latent Variables (NPC-LV), a learning framework for any dataset with abundant unlabeled data but very few labeled ones.
By only training a generative model in an unsupervised way, the framework utilizes the data distribution to build a compressor. Using a compressor-based distance metric derived from Kolmogorov complexity, together with few labeled data, NPC-LV classifies without further training.
We show that NPC-LV outperforms supervised methods on all three datasets on image classification in low data regime and even outperform semi-supervised learning methods on CIFAR-10. We demonstrate how and when negative evidence lowerbound (nELBO) can be used as an approximate compressed length for classification. By revealing the correlation between compression rate and classification accuracy, we illustrate that under NPC-LV, the improvement of generative models can enhance downstream classification accuracy.

\end{abstract}

\section{Introduction}

The progress of deep neural networks drives great success of supervised learning with huge labeled datasets~\cite{simonyan2014very, sutskever2014sequence, gilmer2017neural}. 
However, large labeled datasets are luxurious in many applications and huge amounts of training parameters make the model easy to overfit and hard to generalize to other datasets. 
The urge to learn with small labeled dataset prompts Few-Shot Learning. However, most few-shot classification settings require either an auxiliary ``support set''~\cite{vinyals2016matching,edwards2016towards, snell2017prototypical, sung2018learning} that contains $c$ classes, each has $k$ samples ($c$-way $k$-shot); or prior knowledge about the dataset, where data augmentation can be performed within the same dataset~\cite{kwitt2016one, hariharan2017low, schwartz2018delta} or from other weakly-labeled/unlabeled/similar datasets~\cite{pfister2014domain, douze2018low, gao2018low}. 
This setting is not widely applicable to every dataset in practice, as it requires either an elaborate construction of additional ``support set'' or 
augmentation algorithms tailored to specific datasets. Pre-trained models, on the other hand, do not require adhoc ``support'' and are proved to be good at few-shot learning~\cite{brown2020language} and  even zero-shot learning~\cite{puri2019zero}. However, thousands of millions of training parameters make the model hard to be retrained but only fine-tuned. When the data distribution is substantially different from any datasets used in pre-training, the inductive bias from pre-training holds up fine-tuning, making the model less pliable~\cite{xie2020unsupervised}.

Goals from the above learning paradigms can be summarized as to design algorithms that can be applied to any dataset, and can learn with few labeled data, ideally with no training. ``No Free Lunch''~\cite{wolpert1997no} implies that it's impossible to have an algorithm that is both ``universal'' and ``best''. But how good can a ``universal'' algorithm be, especially in the low data regime, with no external data resources?
Specifically, we are interested in a new setting, \textit{Non-Supported Few-Shot Learning} (NS-FSL), defined as follows:\\
Given any target dataset $\mathbf{D}=(\xv_1, \xv_2, ..., \xv_n)$ belonging to $c$ classes. For each class, we have $k$ labeled samples $(1\leq k\leq 10)$. The remaining $n-ck$ unlabeled samples need to be classified into $c$ classes without the need of support sets, any other datasets or training parameters. 

This setting is similar to semi-supervised learning's but excludes labeled information in training.~\citet{ravi2016optimization} demonstrate that it's hard to optimize a neural network when labeled data is scarce. In order to make minimum assumption about labeled data, we aim at using parameter-free methods. 
The goal is to grasp the data-specific distribution $p(\xv)$, with minimal premises on conditional probability distribution $p(y|\xv)$. Deep generative models with explicit density estimation are perfect candidates for this goal. 
The problem then becomes: given trained generative models, how to take full advantage of the information obtained from them for classification? Using latent representation only utilizes $p(\zv|\xv)$, which just includes partial information. Even for those latent generative models that do not suffer from posterior collapse~\cite{bowman2016generating}, $p(\zv|\xv)$'s insufficiency for classification with non-parametric methods like $k$-nearest-neighbor is shown in both previous work~\cite{zhao2019infovae,davidson2018hyperspherical} and our experiments. 

Inspired by previous work that uses compressor-based distance metrics for non-parametric learning~\cite{chen1999compression, keogh2004towards, chen2004shared,cebrian2005common}, 
we propose Non-Parametric learning by Compression with Latent Variables (NPC-LV), a learning framework that consists of deep-generative-model-based compressors and compressor-based distance metrics. It leverages the power of deep generative models without exploiting any label information in the probabilistic modeling procedure. With no further training, this framework can be directly used for classification. 
By separating probabilistic modeling from downstream tasks that require labels, we grasp the unique underlying structures of the data in every dataset, and further utilize these structures in downstream tasks with no parameter tuning needed. 
We view this learning framework as a \textit{baseline} in this setting, for any dataset. We argue it is ``parameter-free'' as there is no parameter involved in classification stage for labeled data. Basically it means training a generative model as is and getting a classifier for free.

Our contributions are as follows: 
(1) We frame the existing methods into a general learning framework NPC, based on which, we derive NPC-LV, a flexible learning framework with replaceable modules. (2) We use NPC-LV 
as a baseline for a common learning scenario NS-FSL with neither support sets nor further training.
(3) Our method outperforms supervised methods by up to 11.8\% on MNIST, 18.0\% on FashionMNIST, 59\% on CIFAR-10 on image classification in low-data regime. It outperforms non-parametric methods using latent representation on all three datasets. It even outperforms semi-supervised learning methods on CIFAR-10.
(4) We show how negative evidence lowerbound (nELBO) can be used for classification under this framework.
(5) We find the correlation between bitrate and classification accuracy.
This finding suggests the improvement in the domain of deep-learning-based-compressor can further boost classification accuracy under this framework.

\section{Background}

\subsection{Information Theory --- Data Compression}
\label{bg:it}
In a compression scenario, suppose we have a sender \textit{Alice} and a receiver \textit{Bob}. \textit{Alice} wants to send a message that contains a sequence of symbols $\xv=(x_1, x_2, ..., x_n)$ to \textit{Bob}. The ultimate goal of the lossless compressor is to compress $\xv$ into the minimum amount of bits $\xv'$ that can later be decompressed back to $\xv$ by \textit{Bob}. To achieve the shortest compressed length, shorter codes are assigned to symbols with higher probability. According to Shannon's Source Coding Theorem~\cite{shannon1948mathematical}, this length of bits is no shorter than the entropy of the sequence, whose definition is $H(\xv)\triangleq \mathbb{E}[-\log p_{\text{data}}(\xv)]$, where $p_{\text{data}}(\xv)$ represents the probability distribution of each symbol in the sequence. 
Instead of coding symbols one by one, stream code algorithms like Asymmetric Numeral Systems (ANS)~\cite{duda2009asymmetric} convert $\xv$ to a sequence of bits $\xv'$ and reaches this optimal code length for the whole sequence with overhead of around 2 bits, given $p_{\text{data}}(\xv)$. 
However, the ``true'' probabilistic distribution $p_{\text{data}}(\xv)$ is unknown to us. We can only access samples and approximate it with $p_\theta(\xv)$. That is:
\begin{equation}
    \mathbb{E}[-\log p_\theta(\xv)]\geq H(\xv)\triangleq \mathbb{E}[-\log p_{\text{data}}(\xv)].
\end{equation}
Given an entropy coding scheme, the better $p_\theta(\xv)$ approximates $p_{\text{data}}(\xv)$, the closer we can get to the minimum code length. Modeling $p_\theta(\xv)$ is where deep generative model with density estimation can help.
Possible coding schemes and generative models for compressors are discussed in~\ref{sec:bbans}.

\subsection{Algorithmic Information Theory --- Kolmogorov Complexity and Information Distance}
\label{bg:ait}
While information theory is built on data distribution, algorithmic information theory considers ``single'' objects without notion of probability distribution. Kolmogorov complexity $K(x)$~\cite{kolmogorov1963tables} is used to describe the length of the shortest binary program that can produce $x$ on a universal computer, which is the ultimate lower bound of information measurement. Similarly, the Kolmogorov complexity of $x$ given $y$ is the length of the binary program that on input $y$ outputs $x$, denoted as $K(x|y)$. Based on Kolmogorov complexity, ~\citet{bennett1998information} derive \textit{information distance} $E(x,y)$:

\vspace{-.8em}
\begin{equation}
        E(x,y) = \max \{K(x|y), K(y|x)\} = K(xy)-\min\{K(x), K(y)\}.\footnote{Note that $K(xy)$ denotes the length of the shortest program computing $x$ and $y$ without telling which one is which (i.e., no seperator encoded between $x$ and $y$).}
         \label{id}
\end{equation}
The idea behind this measurement, on a high level, is that the similarity between two objects indicates the existence of a simple program that can convert one to another. The simpler the converting program is, the more similar the objects are. For example, the negative of an image is very similar to the original one as the transformation can be simply described as ``inverting the color of the image''.

\begin{theorem}
\label{theo:id}
\textit{
The function $E(x,y)$ is an admissible distance and a metric. It is minimal in the sense that for every admissible distance $D$, we have $E(x,y)\leq D(x,y)+O(1)$.
}
\end{theorem}
Intuitively, \textit{admissible distance} refers to distances that are meaningful (e.g., excluding metrics like $D(x,y)=0.3$ for any $x\neq y$) and computable (formal definition is in~\Cref{sec:norm}). Combining those definitions, we can see~\Cref{theo:id} means $E(x,y)$ is \textit{universal} in a way that it is optimal and can discover all effective similarities between two objects (proof is shown in~\Cref{proof:uni}). 

In order to compare the similarity, relative distance is preferred. \citet{li2001information} propose a normalized version of $E(x,y)$ called \textit{Normalized Information Distance} (NID). 
\begin{definition}[\textbf{NID}]
NID is a function: $ \Omega\times\Omega \rightarrow [0,1]$, where $\Omega$ is a non-empty set, defined as:
\begin{equation}
    \text{NID}(x,y) = \frac{\max \{K(x|y), K(y|x)\}}{\max \{K(x), K(y)\}}.
    \label{nid}
\end{equation}
\end{definition}
\Cref{nid} can be interpreted as follows: Given two sequences $x$, $y$, $K(y)\geq K(x)$:
\begin{equation}
    \text{NID}(x,y) = \frac{K(y)-I(x:y)}{K(y)} = 1 - \frac{I(x:y)}{K(y)},
    \label{nid2}
\end{equation}
where $I(x:y)=K(y)-K(y|x)$ means the \textit{mutual algorithmic information}. $\frac{I(x:y)}{K(y)}$ means the shared information (in bits) per bit of information contained in the most informative sequence, and ~\Cref{nid2} here is a specific case of~\Cref{nid}. Theoretically, NID is a desirable distance metric as it satisfies the metric (in)equalities (definition in~\Cref{sec:norm}) up to additive precision $O(1/K(\cdot))$ where $K(\cdot)$ is the maximum complexities of objects involved in (in)equalities (proof shown in~\cite{li2004similarity}). 


\section{Non-Parametric learning by Compression with Latent Variables}
Non-Parametric learning by Compression (NPC) consists of three modules --- a distance metric, a compressor and an aggregation method shown in~\Cref{fig:npc}. NPC-LV leverages NPC by including neural compressors based on deep generative models. We introduce the derivation of compressor-based distance metrics in~\Cref{sec:kc}; the generative-model-based compressor in~\Cref{sec:bbans}; an integration of this framework with generative models in~\Cref{sec:npcdgm}.

\subsection{Compressor-based Distance Metric}
\label{sec:kc}




Universal as NID is, it is uncomputable, because Kolmogorov complexity is uncomputable.
\citet{cilibrasi2005clustering} propose \textit{Normalized Compression Distance} (NCD), a quasi-universal distance metric based on real-world compressors. In this context, $K(x)$ can be viewed as the length of $x$ after being maximally compressed. Suppose we have $C(x)$ as the length of compressed $x$ produced by a real-world compressor, then NCD is defined as:
\begin{equation}
    \text{NCD}(x, y) = \frac{C(xy) - \min \{C(x), C(y)\}}{\max\{C(x), C(y)\}}.
    \label{ncd}
\end{equation}
The better the compressor is, the closer NCD approximates NID. With a \textit{normal} compressor (discussed in details in~\Cref{sec:norm}), NCD has values in [0,1] and satisfies the distance metric (in)equalities up to $O(\log n/n)$ where $n$ means the maximum binary length of a string involved~\cite{vitanyi2009normalized}. 
NCD is thus computable in that it not only uses compressed length to approximate $K(x)$ but also replaces conditional Kolmogorov complexity with $C(xy)$ that only needs a simple concatenation of $x,y$.
\citet{li2004similarity} simplify NCD by proposing another \textit{compression-based dissimilarity measurement} (CDM):
\begin{equation}
    \text{CDM}(x,y) = \frac{C(xy)}{C(x)+C(y)}.
\end{equation}

\citet{chen2004shared} use another variation ranging from 0 to 1: 
\begin{equation}
    \text{CLM}(x,y)=1-\frac{C(x)+C(y)-C(xy)}{C(xy)}.
\end{equation}
 NCD, CDM and CLM are different variations of Kolmogorov based distance metrics. We empirically evaluate their performance in~\Cref{sec:npc}.

\begin{figure*}[t]
    \centering
    \includegraphics[width=1\linewidth]{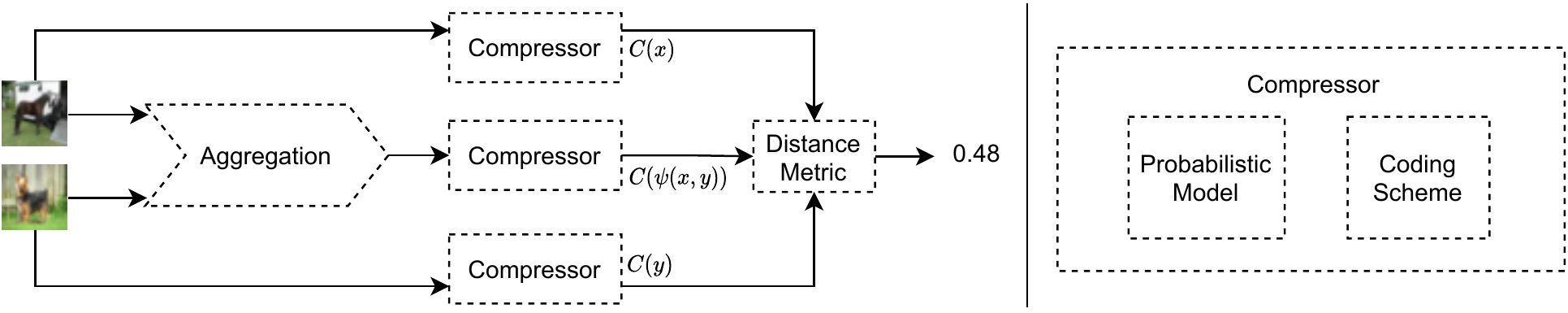}
    \caption{NPC framework with trainable deep probabilistic models. Replaceable modules are indicated with dashed lines.}
    \label{fig:npc}
\end{figure*}

\begin{wrapfigure}{r}{0.5\textwidth}
  \begin{center}
    \includegraphics[width=0.48\textwidth]{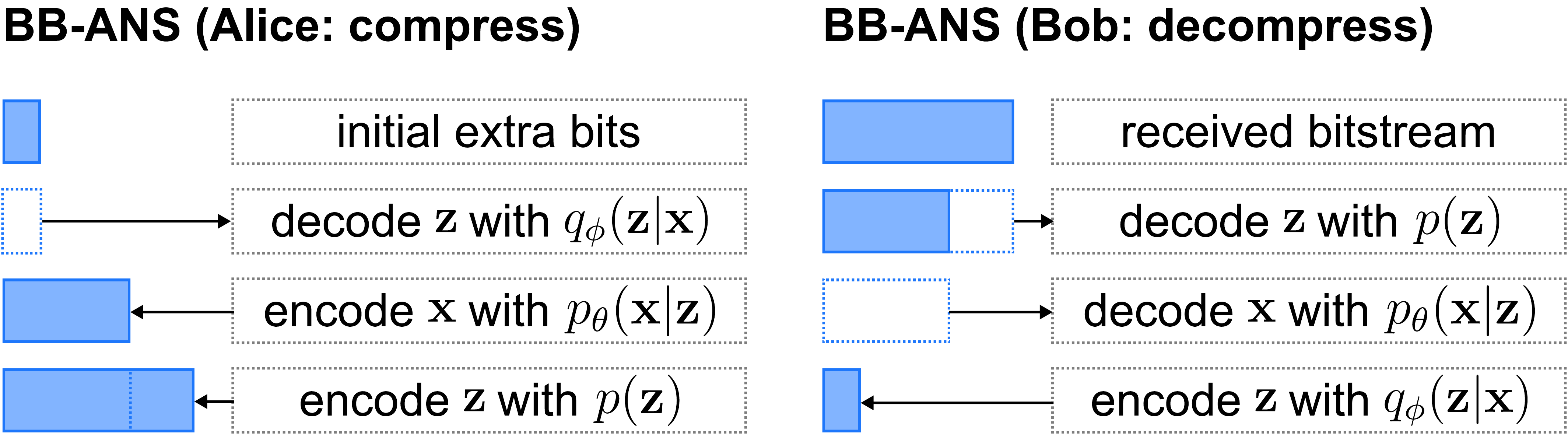}
  \end{center}
    \caption{BB-ANS compress \& decompress}
    \label{fig:bbans}
\end{wrapfigure}

\subsection{Trained Generative Models as Compressors}
\label{sec:bbans}
Previous works~\cite{chen2004shared, bennett1998information, cilibrasi2005clustering} demonstrate the success of compression-based distance metrics in sequential datasets like time series, DNA, and texts using non-neural compressors like gzip, bzip2, PPMZ.
Deep-generative-model–based compressors can take NPC to the next level by expanding to more data types using better compressors. We mainly focus on variational autoencoder (VAE) based compressors with brief introduction to other neural compressors.

\textbf{VAE Family: } The relation between VAE and ``bits-back'' has been revealed in multiple previous work~\cite{chen2016variational,honkela2004variational}. \citet{townsend2018practical} use latent variable models by connecting ANS to ``bits-back'' argument~\cite{frey1997efficient} (BB-ANS). 
In the setting of ``bits-back argument'', we assume \textit{Alice} has some extra bits of information to send to \textit{Bob} alongside with $\xv$. It's also assumed that both \textit{Alice} and \textit{Bob} have access to $p(\zv)$, $p_\theta(\xv|\zv)$ and $q_\phi(\zv|\xv)$ where $\zv$ is the latent variable; $p(\zv)$ is the prior distribution of $\zv$; $p_\theta(\xv|\zv)$ represents a generative network and $q_\phi(\zv|\xv)$ represents an inference network. 
As shown in the~\Cref{fig:bbans}, 
\textit{Alice} first decodes those extra information according to $q_\phi(\zv|\xv)$ to generate a sample $\zv$.\footnote{Note that ``encoding'', ``decoding'' here follows data compression's convention instead of variational autoencoder's.} $\zv$ is further used to encode $\xv$ with $p_\theta(\xv|\zv)$ and $\zv$ itself is encoded using $p(\zv)$. \textit{Bob} then reverses this procedure and recovers the extra bits by encoding with $q_\phi(\zv|\xv)$. For a single data point, the length of the final bitstream is:
\begin{equation}
    N = n_{\text{extra}} + \log q_\phi(\zv|\xv) - \log p_\theta(\xv|\zv) - \log p(\zv).
\end{equation}
We can see the expectation of $N-n_{\text{extra}}$ is equal to the negative evidence lower bound (nELBO):

\begin{equation}
     \mathbb{E}_{q_\phi(\zv|\xv)}[N-n_{\text{extra}}] = -\mathbb{E}_{q_\phi(\zv|\xv)}\log \frac{p_\theta(\xv,\zv)}{q_\phi(\zv|\xv)} = -\text{ELBO}
\label{eq:elbo}
\end{equation}



ELBO above is derived from ``bits-back argument'' in the context of compression. Now, from the perspective of latent variable model like VAE, the derivation often starts from the fact that $p_\theta(\xv) = \int p_\theta(\xv|\zv)p(\zv)$ is intractable. $q_\phi(\zv|\xv)$ is then introduced as an inference model to approximate $p(\zv|\xv)$ in order to work around the intractability problem, which brings up the marginal log likelihood:
\vspace{-0.5em}
\begin{equation}
\begin{split}
    & \log p_\theta(\xv)  = \mathbb{E}_{q_\phi(\zv|\xv)}\log \frac{p_\theta(\xv,\zv)}{q_\phi(\zv|\xv)} + \mathbb{E}_{q_\phi(\zv|\xv)}\log \frac{q_\phi(\zv|\xv)}{p(\zv|\xv)}, \\
    & \text{ ELBO}  = \log p_\theta(\xv)-D[q_\phi(\zv|\xv)\| p(\zv|\xv)].
\end{split}
\end{equation}


We only need to optimize the lower bound, as minimizing nELBO means maximizing $\log p_\theta(\xv)$ --- the likelihood of generating real data and minimizing KL divergence between $q_\phi(\zv|\xv)$ and $p(\zv|\xv)$ at the same time, which is the same objective function from what we derive using ``bits-back''.

This equivalence demonstrates that an optimized latent variable model can be used directly for compression as, from the data compression perspective, it minimizes the code length attainable by bits back coding using the model. With the help of ANS, we can encode symbols into bitstreams or decode bitstreams back to symbols with trained latent variable models. Details of ANS and discretizing continuous latent variables are shown in~\Cref{appx:ans,appx:disc}. 

\textbf{ARM: } Autoregressive models (ARM) model $p(\xv)$ as: 
    $p(\xv)=p(x_0)\prod_{i=1}^n p(x_i|\xv_{i-1})$.
The exact likelihood estimation makes it capable of lossless compression. But instead of using ANS, which is a stack-like coding scheme, queue-like ones (e.g., Arithmetic Coding (AC)~\cite{witten1987arithmetic}) should be used. Computational inefficiency is the main drawback of ARM such as RNN~\cite{rumelhart1985learning}, but causal convolutions~\cite{van2016conditional, van2016pixel} can alleviate the problem. 

\textbf{IDF: } Integer Discrete Flow (IDF)~\cite{hoogeboom2019integer} can also optimize towards the exact log-likelihood. Similar to other flow-based models, it utilizes invertible transformation $f$, but works on discrete distribution with additive coupling and rounding operation. For IDF, ANS can be used as the entropy coder.

Our experiments and discussion are mainly about VAE-based compressors as their architectures can be flexibly changed under the ``bits-back'' argument.

\subsection{NPC-LV}
\label{sec:npcdgm}



We’ve shown in~\Cref{sec:bbans} that we can plug in any trained latent variable model in exchange for a near optimal compressor under the framework of BB-ANS. To show that the coding scheme is replaceable, we introduce a coding scheme variation, Bit-Swap~\cite{kingma2019bit}, in the following experiments. The difference between BB-ANS and Bit-Swap is the encoding and decoding order when there are more than one latent variable. A detailed comparison is shown in~\Cref{appx:init}. 
The generative model we use for both two is a hierarchical latent generative model (details are in~\Cref{appx:hier}), also known as Deep Latent Gaussian Model (DLGM)~\cite{rezende2014stochastic}. 
\begin{algorithm}[h]
\caption{NPC-LV (use VAE and NCD as an example)}
\label{alg:npclv}
\begin{algorithmic}
\STATE {\bfseries Input:} $k$, $\mathbf{D}_\text{test}$,$\mathbf{D}_\text{train}=\{\mathbf{D}_\text{train}^U, \mathbf{D}_\text{train}^L\}$, $\mathbf{D}_\text{train}^L=\{X^L, y^L\}$\\

\STATE trained\_vae = \textcolor{teal}{trainVAE}($\mathbf{D}_\text{train}^U$) \\

\STATE vae\_compressor = \textcolor{teal}{ANS}(trained\_vae) \\
\FOR{$\xv_\text{test}$ in $\mathbf{D}_\text{test}$}
    \STATE C\_$\xv_\text{test}$ = \textcolor{teal}{len}(vae\_compressor($\xv_\text{test}$)) \\
    \STATE distances = []\\
    \FOR{$\xv_\text{train}$ in $X^L$}
        \STATE \textit{\textcolor{lightgray}{// Calculate NCD distance with $C(x)$, $C(y)$, and $C(\psi(x,y))$}} \\
        \STATE C\_$\xv_\text{train}$ = \textcolor{teal}{len}(vae\_compressor($\xv_\text{train}$)) \\
        \STATE C\_agg = \textcolor{teal}{len}(vae\_compressor( \textcolor{teal}{aggregate}($\xv_\text{test}$, $\xv_\text{train}$)) \\
        \STATE NCD = (C\_agg $-$ \textcolor{teal}{min}\{C\_$\xv_\text{test}$, C\_$\xv_\text{train}$\}) $/$ \textcolor{teal}{max}\{C\_$\xv_\text{test}$, C\_$\xv_\text{train}$\} \\
        \STATE distances = \textcolor{teal}{push}(NCD, distances)
    \ENDFOR \\
    \STATE k\_nearest\_indicies =  \textcolor{teal}{argsort}(distances)[:$k$] \\
    \STATE $y_\text{test}$ = \textcolor{teal}{majority}($\{y^L_i, i\in\text{k\_nearest\_indicies}\})$
\ENDFOR
\end{algorithmic}
\end{algorithm}

\textbf{Aggregation: }In addition to coding schemes and probabilistic models for compressors, aggregation methods can also be replaced.
Previous works~\cite{cilibrasi2005clustering} assume $xy$ in $C(xy)$ means the ``concatenation'' of two inputs. We expand this assumption to a more general case where ``aggregation'' can be other kind of aggregation function, represented as $C(\psi(x,y))$ in~\Cref{fig:npc}. We prove the legitimacy of this generalization as changing aggregations methods may make compressor-based distance metrics \textit{not admissible} (shown in~\Cref{sec:norm}).
More sophisticated strategies of aggregation are left to the future work. 
We also discuss other replaceable modules in details in~\Cref{sec:factor}.

We use BB-ANS with NCD as a concrete instance to demonstrate this framework on a classification task shown in~\Cref{alg:npclv}. $\mathbf{D}_\text{train}^U$ and $\mathbf{D}_\text{train}^L$ mean unlabeled and labeled training set; predefined functions are in teal. The algorithm can be simplified into four steps: 1) Train a VAE on the unlabeled training dataset; 2) Apply ANS with discretization on the trained VAE to get a compressor; 3) Calculate the distance matrix between pairs $(\xv_\text{test}, \xv_\text{train})$ with the compressor and NCD; 4) Run $k$-Nearest-Neighbor($k$NN) classifier with the distance matrix.

\textbf{nELBO as estimated compressed length: }Specifically for VAE-based compressors, nELBO is the estimated length of the compressed bitstream as~\Cref{eq:elbo} shows. Therefore, we can use it directly without actual compressing. This can further simplify our method as we don't need to discretize continuous distribution or apply entropy coder. 

The reason why the underlying data distribution can help the classification is based on \textit{manifold assumption}, which is a common assumption in SSL~\cite{van2020survey}. It states that the input space consists of multiple low-dimensional manifolds, and data points lying on the same manifold have same labels. This assumption helps alleviate the curse of dimensionality and may explain the effectiveness of using $k$NN with few labeled data. Due to the fact that our training process does not use labeled data, our method does not rely on other common assumptions in SSL like~\textit{smooth assumption} and \textit{low-density assumption}. 
The facts that NPC-LV makes very few assumptions about datasets and that compressors are data-type-agnostic make this framework extensible to other data types beyond images. For example, a combination of an autoregressive model (e.g., character recursive neural network) and arithmetic coding~\cite{goyal2019deepzip} can be used in our framework for sequential data.

\section{Related Work}



\textbf{Non-parametric learning with \textit{Information Distance}: }
\label{sec:rw}
\citet{bennett1998information} propose \textit{information distance} as a universal metric, 
based on which, several papers~\cite{grumbach1994new, krebsforschungzentrumtransformation, yianilos2002normalized} propose more fine-grained distance metrics. 
\citet{cilibrasi2005clustering,li2004similarity,chen2004shared} derive more practical distance metrics based on real-world compressors.
Empirical results~\cite{chen1999compression, keogh2004towards, chen2004shared,cebrian2005common} show that even without any training parameters, those compressor-based distance metrics can produce effective distance matrix for clustering and classification on time series datasets.
\citet{cilibrasi2005clustering} further push this direction to more types of datasets including images that are represented in ``\#'' (black pixel) and ``.'' (white pixel). We unify previous work in NPC framework, expand it to real image datasets and leverage it with neural compressors.

\textbf{Compression:} 
\citet{shannon1948mathematical} establishes source coding theorems, showing that entropy rate is the limit of code rate for lossless compression. Huffman Coding~\cite{huffman1952method} achieves the optimal symbol code whose length is upper-bounded by $H(\xv)+1$ per symbol. The 1 bit overhead is due to the fact that $-\log p(x)$ is not always an integer. Stream coding like AC~\cite{witten1987arithmetic} and ANS~\cite{duda2009asymmetric} further optimize by representing the whole message with numeral systems. 
Those entropy coders then can be combined with probabilistic modeling using neural network~\cite{schmidhuber1996sequential,mahoney2000fast,goyal2019deepzip} and used in our framework. 

\textbf{Semi-Supervised Learning with VAE: }
The evaluation paradigm in this paper is closest to Semi-Supervised Learning (SSL). \citet{kingma2014semi} design two frameworks for utilizing autoencoders in downstream classification tasks. The first (M1) is to train a tSVM~\cite{vapnik1999nature} with latent representation output by a trained VAE. The second (M2) is to train a VAE with label as another latent variable. M1 only requires a standard VAE but tSVM suffers from optimization difficulty~\cite{collobert2006large}, making it hard to be generally applicable for VAE. 
Later VAE-based methods~\cite{maaloe2016auxiliary,joy2020rethinking} are built on M2. These methods 
don't train a generative model in an unsupervised way like we do. 

\textbf{Few-Shot Learning: }Similar to our setting, FSL also targets at low labeled data regime. Large amounts of previous works~\cite{vinyals2016matching, snell2017prototypical,edwards2016towards,sung2018learning} on FSL are based on meta-learning, where the model is fed with an extra labeled support set, in addition to the target dataset. Another line of work~\cite{pfister2014domain, kwitt2016one, hariharan2017low, schwartz2018delta, douze2018low, gao2018low} utilize data augmentation. Although some of them do not require extra dataset~\cite{kwitt2016one, hariharan2017low, schwartz2018delta}, the augmentation algorithms can hardly be applied to every other dataset~\cite{wang2020generalizing}. Metric-based methods~\cite{koch2015siamese} utilize distance metrics for FSL. But instead of modeling the probability distribution of a dataset, they model the ``distance'' between any pair of data points with neural network, which still requires many labeled data during “pre-training”. Our work is similar to metric-based methods in that both have the essence of nearest-neighbor. The difference is that in the ``pre-training'' stage, our model is not trained to learn the distance but to reconstruct the image as all standard generative models do. More importantly, we use no labeled data in this stage.

\begin{table}[t!]

\centering
\resizebox{\textwidth}{!}{%
\begin{tabular}{c|c|c|c|c|c|c|c|c|c}

    Data & \multicolumn{3}{c|}{MNIST} & \multicolumn{3}{c|}{FashionMNIST} & \multicolumn{3}{c}{CIFAR-10}  \\
    \hline
    \#Shot & 5 & 10 & 50 & 5 & 10 & 50 & 5 & 10 & 50 \\
    \toprule
    \multicolumn{10}{c}{\textbf{Supervised Learning}} \\
    \hline
    SVM & 69.4\sbpm{2.2} & 77.1\sbpm{1.5} & 87.6\sbpm{0.4} & 67.1\sbpm{2.1} & 71.0\sbpm{1.6} & 78.4\sbpm{0.5} & 21.1\sbpm{1.9} & 23.6\sbpm{0.5} & 27.2\sbpm{1.2} \\
    \cline{2-10}
    \#Param & \multicolumn{3}{c|}{\scalebox{0.8}{35,280}} & \multicolumn{3}{c|}{\scalebox{0.8}{35,280}} & \multicolumn{3}{c}{\scalebox{0.8}{105,840}} \\
    \hline 
    CNN & 72.4\sbpm{3.5} & 83.7\sbpm{2.6} & 93.2\sbpm{2.8} & 67.4\sbpm{1.9} & 70.6\sbpm{2.5} & 80.5\sbpm{0.7} & 23.4\sbpm{2.9} & 28.3\sbpm{1.9} & 38.7\sbpm{1.9}  \\
    \cline{2-10}
    \#Param & \multicolumn{3}{c|}{\scalebox{0.8}{1,199,882}} & \multicolumn{3}{c|}{\scalebox{0.8}{1,199,882}} & \multicolumn{3}{c}{\scalebox{0.8}{1,626,442}} \\
    \hline
    VGG & 69.4\sbpm{5.7} & 83.9\sbpm{3.2} & 94.4\sbpm{0.6} & 62.8\sbpm{4.1} & 70.5\sbpm{4.5} & 81.5\sbpm{1.1} & 22.2\sbpm{1.6} & 29.7\sbpm{1.8} & 42.6\sbpm{1.2} \\
    \cline{2-10}
    \#Param & \multicolumn{3}{c|}{\scalebox{0.8}{28,148,362}} & \multicolumn{3}{c|}{\scalebox{0.8}{28,148,362}} & \multicolumn{3}{c}{\scalebox{0.8}{28,149,514}} \\
    \toprule
    \multicolumn{10}{c}{\textbf{Semi-Supervised Learning}} \\
    \hline
    VAT & \textit{97.0\sbpm{0.3}} & \textit{97.4\sbpm{0.1}} & \textit{98.4\sbpm{0.1}} & \textit{74.1\sbpm{0.8}} & \textit{78.4\sbpm{0.3}} & \textit{87.1\sbpm{0.2}} & 25.4\sbpm{2.0} & 27.8\sbpm{4.2} & 60.9\sbpm{6.1} \\
    \cline{2-10}
    \#Param & \multicolumn{3}{c|}{\scalebox{0.8}{1,469,354}} & \multicolumn{3}{c|}{\scalebox{0.8}{1,469,354}} & \multicolumn{3}{c}{\scalebox{0.8}{1,469,642}} \\
    \hline
    MT & 78.4\scalebox{0.7}{$\pm$2.0} & 82.8\sbpm{1.9} & 98.6\sbpm{0.2} & 58.1\sbpm{2.8} & 70.8\sbpm{0.8} & \textit{87.1\sbpm{0.1}} & 31.7\sbpm{1.5} & 35.9\sbpm{1.1} & \textit{64.3\sbpm{1.6}} \\
    \cline{2-10}
    \#Param & \multicolumn{3}{c|}{\scalebox{0.8}{1,469,354}} & \multicolumn{3}{c|}{\scalebox{0.8}{1,469,354}} & \multicolumn{3}{c}{\scalebox{0.8}{1,469,642}} \\
    \toprule
    \multicolumn{10}{c}{\textbf{Non-Parametric Learning}} \\
    \hline
    Single & 65.6\sbpm{1.2} & 76.8\sbpm{0.8} & 86.3\sbpm{0.3} & 40.2\sbpm{1.4} & 53.4\sbpm{1.1} & 70.0\sbpm{0.4} & 17.3\sbpm{0.9} & 19.2\sbpm{0.7} & 23.4\sbpm{0.3}  \\
    \cline{2-10}
    \#Param & \multicolumn{3}{c|}{\scalebox{0.8}{0}} & \multicolumn{3}{c|}{\scalebox{0.8}{0}} & \multicolumn{3}{c}{\scalebox{0.8}{0}} \\
        \hline
    Hier & 73.6\sbpm{3.1} & 82.3\sbpm{2.1} & 90.4\sbpm{1.4}  & 69.5\sbpm{3.5} & 72.5\sbpm{1.9} & 78.7\sbpm{1.3}  & 22.2\sbpm{1.6} & 24.2\sbpm{4.9} & 26.2\sbpm{2.9}  \\
    \cline{2-10}
    \#Param & \multicolumn{3}{c|}{\scalebox{0.8}{0}} & \multicolumn{3}{c|}{\scalebox{0.8}{0}} & \multicolumn{3}{c}{\scalebox{0.8}{0}} \\
    \toprule
    \multicolumn{10}{c}{\textbf{Non-Parametric learning by Compression with Latent Variables (NPC-LV)}} \\
    \hline
    \scalebox{0.9}{nELBO} & \textbf{75.2\sbpm{1.5}} & 81.4\sbpm{1.1} & 91.0\sbpm{1.0}  & \textbf{72.2\sbpm{2.2}} & \textbf{76.7\sbpm{1.5}} & \textbf{85.6\sbpm{1.1}} & \underline{\textbf{34.1\sbpm{1.8}}} & \textbf{34.6\sbpm{2.0}} & 35.6\sbpm{2.5}  \\
    \cline{2-10}
    \#Param & \multicolumn{3}{c|}{\scalebox{0.8}{0}} & \multicolumn{3}{c|}{\scalebox{0.8}{0}} & \multicolumn{3}{c}{\scalebox{0.8}{0}} \\
    \hline
    \scalebox{0.9}{Bit-Swap} & \textbf{75.7\sbpm{3.6}} & 83.3\sbpm{0.9} & 90.9\sbpm{0.2} & \textbf{73.5\sbpm{3.7}} & \textbf{76.0\sbpm{1.4}} & \textbf{82.6\sbpm{1.2}} & \underline{\textbf{32.2\sbpm{3.5}}} & \textbf{32.8\sbpm{1.9}} & 35.7\sbpm{1.1} \\
    \cline{2-10}
    \#Param & \multicolumn{3}{c|}{\scalebox{0.8}{0}} & \multicolumn{3}{c|}{\scalebox{0.8}{0}} & \multicolumn{3}{c}{\scalebox{0.8}{0}} \\
    \hline
    \scalebox{0.9}{BB-ANS} & \textbf{77.6\sbpm{0.4}} & \textbf{84.6\sbpm{2.1}} & 91.4\sbpm{0.6} & \textit{\textbf{74.1\sbpm{3.2}}} & \textbf{77.2\sbpm{2.2}} & \textbf{83.2\sbpm{0.7}} & \textit{\underline{\textbf{35.3\sbpm{2.9}}}} & \textit{\underline{\textbf{36.0\sbpm{1.8}}}} & 37.4\sbpm{1.2} \\
    \cline{2-10}
    \#Param & \multicolumn{3}{c|}{\scalebox{0.8}{0}} & \multicolumn{3}{c|}{\scalebox{0.8}{0}} & \multicolumn{3}{c}{\scalebox{0.8}{0}} \\
    \hline

\end{tabular}}
    \caption{Test accuracy of methods with number of learning parameters for classification. \#Shot refers to the the number of training samples per class. Results report means and 95\% confidence interval over five trials. Note that ``\#Param'' refers to parameters specifically for supervised training. }
    \label{tab:sup}
    \vspace{-1.5em}
\end{table}

\section{Experiments}
\label{sec:exp}

We compare our method with supervised learning, semi-supervised learning, non-parametric learning and traditional NPC on MNIST, FashionMNIST and CIFAR-10~\cite{lecun-mnisthandwrittendigit-2010,xiao2017/online,krizhevsky2009learning}. For each dataset, we first train a hierarchical latent generative model with \textit{unlabeled} training sets. During the stage of calculating distance metric using compression, we pick 1,000 samples from test set, due to the cost of compression and pair-wise computation, together with $n=\{5,10,50\}$ labeled images per class from training sets. We also report the result for $n=50$ although it is beyond our setting. We keep the selected dataset same for every method compared with. We use \textbf{bold} to highlight the cases we outperform supervised methods; use \underline{underline} to highligh the case we outperform SSL and use \textit{italic} to highlight the highest accuracy among all methods for reference. We do actual compression with two coding schemes (BB-ANS and Bit-Swap) and also use nELBO for compressed length directly.


\textbf{Comparison with Supervised Learing: }Supervised models are trained on $10n$ labeled data. In~\Cref{tab:sup}, when $n=50$ CNN and VGG surpass NPC-LV. In the cases where the number of labeled data is extremely limited (e.g., $5,10$ labeled data points per class), however, BB-ANS variant outperforms all methods in every dataset. For MNIST, both BB-ANS and Bit-Swap produce more accurate results than supervised methods on 5-shot experiments; BB-ANS performs slightly better than supervised methods in the 10-shot scenario. For FashionMNIST, all three variants outperform in 5-shot, 10-shot, and even 50-shot settings. For CIFAR-10, given 10 labeled data points per class, NPC-LV boosts the accuracy of CNN by 27.2\% and improves the accuracy of VGG by 21.2\%. This enhancement is more significant in 5-shot setting: NPC-LV improves the accuracy of CNN by 50.9\% and by 59.0\% for VGG. In general, we can see as the labeled data become fewer, NPC-LV becomes more advantageous.

\textbf{Comparison with Semi-Supervised Learning: }
The input of NS-FSL is similar to SSL in that both unlabeled data and labeled data are involved. The difference lies in the fact that 1) our training \textit{doesn't} use any labeled data and is purely unsupervised; 2) we only need to train the model \textit{once}, while SSL need to retrain for different $n$; 3) we use fewer labeled data point, which is a more practical setting for real-world problems. We choose strong semi-supervised methods that make little assumption about the dataset. We use consistency regularization methods instead of pseudo-labeling ones as pseudo-labeling methods often assume that decision boundary should pass through low-density region of the input space (e.g.,~\citet{lee2013pseudo}). Specifically, we choose MeanTeacher (``MT'')~\cite{tarvainenmean} and VAT~\cite{miyato2018virtual}. The core of both is based on the intuition that realistic perturbation of data points shouldn't affect the output. We train both models with $n$ labeled samples per class together with an unlabeled training set (training details are shown in~\Cref{appx:train}). 
As we can see, NPC-LV achieve higher accuracy for CIFAR-10 in low-data regime, has competitive result on FashionMNIST and is much lower on MNIST. The strength of our method is more obvious with more complex dataset. It's a surprising result because we do not implement any data augmentation implicitly or explicitly unlike consistency regularization methods, who utilize data perturbation, which can be viewed as data augmentation.
It's worth noting that, on all three datasets, our method using BB-ANS always outperforms \textbf{\textit{at least one}} semi-supervised methods on 10-shot setting, indicating SSL methods trade ``universality'' for ``performance'' while our method is more like a baseline. Speed-wisely, NPC-LV only requires training once for generative model, and can run $k$nn on different shots ($n$) with no additional cost. In contrast, SSL methods require the whole pipeline to be retrained for every $n$.

\begin{table*}[t]
\centering
\begin{small}
\begin{tabular}{||c|c|c|c|c|c|c|c|c|c||}
    \hline
     & \multicolumn{3}{c|}{MNIST} & \multicolumn{3}{c|}{FashionMNIST} & \multicolumn{3}{c||}{CIFAR-10}  \\
     \hline
     & NCD & CLM & CDM & NCD & CLM & CDM & NCD & CLM & CDM \\
     \hline
    gzip & 86.1 & 85.6 & 85.6 & 81.7 & 82.6 & 82.6 & 31.3& 30.3 & 30.3 \\
    bz2 & 86.8 & 86.4 & 86.4 & 81.7 & 79.0 & 79.0 & 28.0 & 27.5 & 27.5 \\
    lzma & 87.4 & 88.5 & 88.5 & 80.6 & 82.7 & 82.7 & 31.4 & 30.0 & 30.0 \\
    WebP & 86.4 & 87.9 & 87.9 & 69.9 & 67.3 & 67.3 & 33.3 & 34.2 & 34.2 \\
    PNG & 86.8 & 89.1 & 89.1 & 74.8 & 76.9 & 76.9 & 32.2 & 28.9 & 28.9 \\
    BitSwap & 93.2 & 90.9 & 93.2 & 84.3 & \textit{84.0} & \textit{84.0} & 36.9 & 36.9 & 36.9 \\
    BBANS & \textit{93.6} & \textit{93.4} & \textit{93.4} & \textit{84.5} & 83.6 & 83.6 & \textit{40.2} & \textit{40.8} & \textit{40.8} \\
    \hline
\end{tabular}
    \caption{Classification accuracy among different compressors and distance metrics. }
    \vspace{-1.2em}
    \label{tab:compress}
\end{small}
\end{table*}

\textbf{Comparison with Non-Parametric Learning: }
In this experiment, we explore the effectiveness of using latent representations directly with $k$NN. We train the same generative model we used in NPC-LV (``Hier''), as well as a vanilla VAE with a single latent variable (``Single''). \Cref{tab:sup} shows that the latent representation of the vanilla one is not as expressive as the hierarchical one. Although latent representation using the hierarchical architecture performs reasonably well and surpasses supervised methods in 5-shot setting on MNIST and FashionMNIST, it's still significantly lower than NPC-LV in all settings. The result suggests that NPC-LV can utilize trained latent variable models more effectively than simply utilizing latent representation for classification.

\textbf{Comparison with NPC: }
\label{sec:npc}
We investigate how NPC with non-neural compressors perform with different distance metrics. We evaluate with NCD, CLM, and CDM as distance metrics, and gzip, bz2, lzma, WebP, and PNG as compressors, using 1,000 images from test set and 100 samples per class from training set. The result is shown in~\Cref{tab:compress}. 
For distance metrics, we can see CLM and CDM perform similarly well but it's not clear under what circumstances a distance metric is superior to the rest.
For compressors, both Bit-Swap and BB-ANS perform much better than other compressors, indicating that generative-model-based compressors can significantly improve NPC. BB-ANS turns out to be the best compressor for classification on all three datasets.

\section{Analyses and Discussion}

\subsection{nELBO as Compressed Length}
As we've shown in~\ref{sec:bbans}, nELBO can be viewed as the expected length of compressed bitstream $N-n_\text{extra}$. Thus, theoretically it can be used directly to approximate compressed length.
In this way, we don't need to apply ANS to VAE for the actual compression, which largely simplifies the method and boosts the speed. 
However, as we can see in~\Cref{tab:sup}, using nELBO doesn't always perform better than the actual compressor like BB-ANS. This may be because nELBO in a well-trained model regards the aggregation of two images as out-of-distribution data points; while the discretization in the actual compressor forces close probability with a certain level of precision discretized to the same bin, lowering the sensitivity. Better aggregation strategies need to be designed to mitigate the gap. 

\begin{wrapfigure}{r}{0.5\textwidth}
  \begin{center}
    \includegraphics[width=0.42\textwidth]{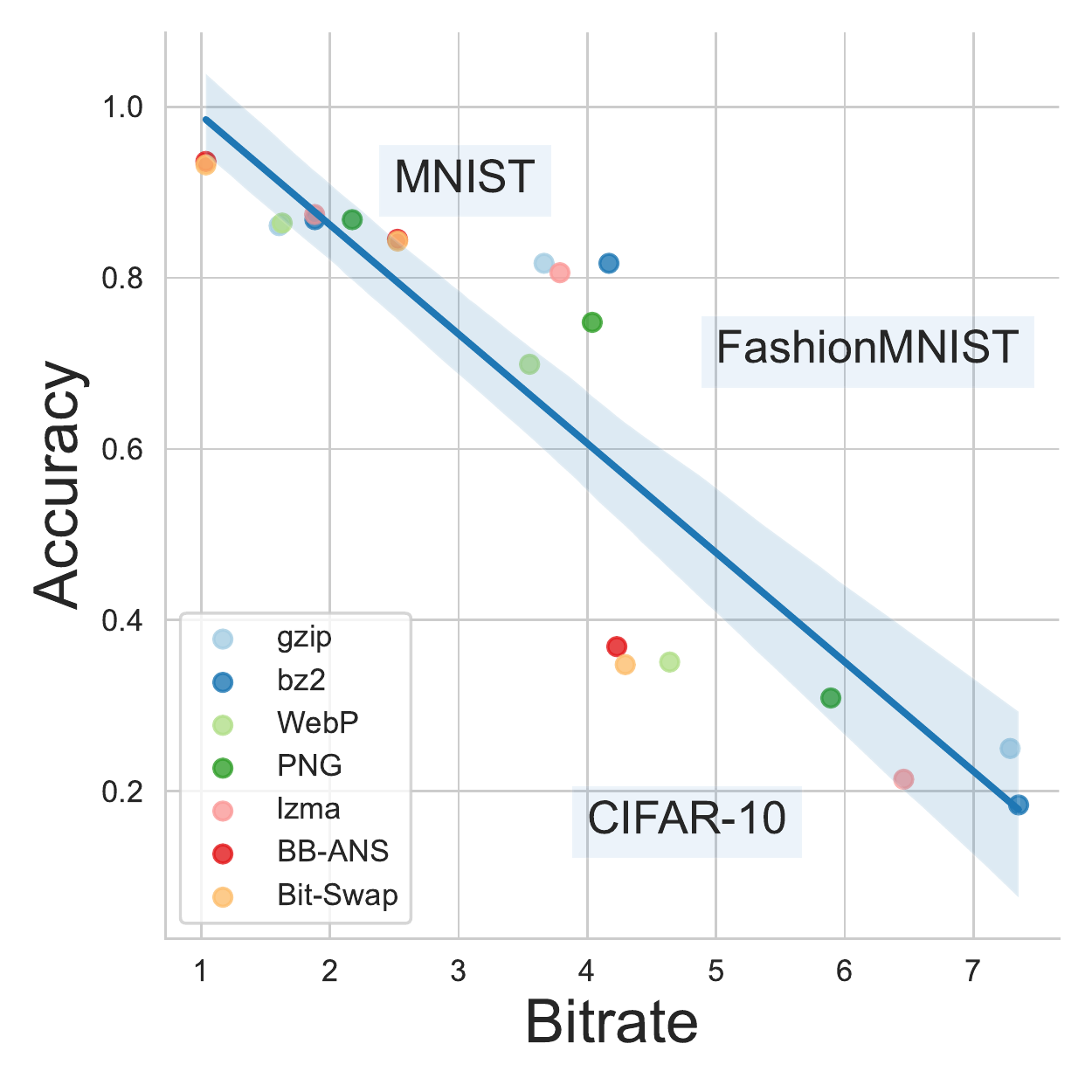}
  \end{center}
    \caption{Bitrate versus Classification Accuracy}
    \label{fig:bit_acc}
\end{wrapfigure}

\subsection{Bitrate versus Classification Accuracy}
\label{sec:bitrate}
The origin of NPC framework comes from the intuition that the length of $x$ after being maximally compressed by a real-world compressor is close to $K(x)$. Theoretically, the closer this length approximates the \textit{minimum} length of the expression ($C(x)\approx K(x)$), the closer the compressor-based distance metrics are to the \textit{normalized information distance}. We investigate, empirically, whether the bitrate actually reflect the classification accuracy. 
We plot bitrate versus classification accuracy for each compressor in~\Cref{tab:compress} on three datasets as shown in~\Cref{fig:bit_acc}. We use the net bitrate, which is $(N-n_\text{extra})/d$, where $N$ is the length of the compressed bitstream, $n_\text{extra}$ is the length of the extra bits, and $d$ is the number of pixels.
As we can see, a very strong monotonic decreasing correlation between bitrate and accuracy emerges, with Spearman coefficient~\cite{spearman1961proof} $r_s=-0.96$, meaning the lower the bitrate is, the higher the classification accuracy is. 
 This means the correlation between bitrate and classification accuracy holds empirically regardless of datasets. It will be interesting to investigate in the future whether the correlation remains for lossy compression.

\subsection{Parallelization and Limitation}
\label{sec:limitation}
The training of generative models can be parallelized using modern GPUs. Compression, however, is not that easy to parallelize. The calculation of CDF and PDF are parallelizable for common probability distributions like Gaussian distribution, but ANS algorithm is not trivial to be parallelized. Fortunately, efficient implementation for ANS on GPUs has been developed~\cite{giesen2014interleaved} to exploit GPU's parallelization power. During distance computation stages, as only pair-wise distance is needed, we can use multi-threads to accelerate the computation.
For classification, we show that NPC-LV performs well in the low labeled data regime, where the complexity of computation may not be a concern, yet the complexity of $O(n^2)$ may still hinder the application involving pair-wise distance computation like clustering on large datasets, unless we exploit the parallelization for compression.

\section{Conclusion}
In this paper, we propose a learning framework Non-Parametric Learning by Compression with Latent Variables (NPC-LV) to address a common learning scenario Non-Supported Few-Shot-Learning (NS-FSL). This framework is versatile in that every module is replaceable, leading to numerous variations. We use image classification as a case study to demonstrate how to use a trained latent generative model directly for downstream classification without further training. It outperforms supervised learning and non-parametric learning on three datasets and semi-supervised learning on CIFAR-10 in low data regime. We thus regard it as a baseline in NS-FSL.
The equivalence between optimizing latent variable models and achieving the shortest code length not only shows how nELBO can be used for classification, but also indicates the improvement of latent probabilistic models can benefit neural compressors. The relationship between compression rate and classification accuracy suggests that the improvement of neural compressors can further benefit classification. Thus, an enhancement of any module in this chain can boost classification accuracy under this framework.

\bibliography{papers}
\bibliographystyle{unsrtnat}

\newpage
\appendix

\section{Ablation Study}
\label{sec:factor}

\begin{figure*}[th]
    \centering
    \includegraphics[width=1\linewidth]{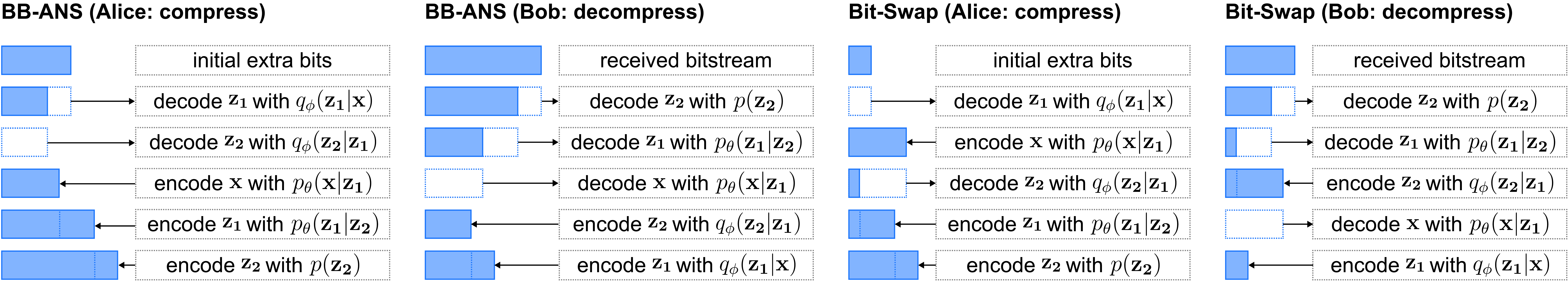}
    \caption{BB-ANS V.S. Bit-Swap}
    \label{fig:bbans-bitswap}
\end{figure*}

\textbf{Number of Latent Variables: }The difference between how BB-ANS~\cite{townsend2018practical} and Bit-Swap~\cite{kingma2019bit} compress is shown in~\Cref{fig:bbans-bitswap}. We can see these two compressors are only different when there are more than one latent variable, and by utilizing encoded bitstreams from previous latent variables, Bit-Swap can reduce the cost of initial bits significantly. We carry out experiments on CIFAR-10 and keep aggregation method constant. We can see in~\Cref{tab:z}, Bit-Swap performs better with more latent variables. But Bit-Swap with eight latent variables is still worse than BB-ANS with two. 

\textbf{Aggregation Methods: }We also evaluate how different aggregation methods affect the classification accuracy, shown in right hand side of~\Cref{tab:z}. We aggregate two images by ``average'', ```minimum'', ``maximum'', ``concatenation'' and ``greyscale+average''. For ``concat'', the actual operation is that we compress one image after another in a way that the compressed bitstream of the first image could be used as the ``extra'' bits for the second image. 
``gs+avg'' means we use greyscale of image and then average pixel values. ``greyscale'' is not a way of aggregation, but we notate it in the table this way for simplicity and comparison, and also to stress this operation is only used during compression.
That is, the generative model is trained on original images instead of on greyscale images. 
By applying simple image processing methods during compression, we find that combining aggregation with image manipulation can be effective as no compressor need to be changed, and thus no retraining for generative models is needed.
We also plot their bitrate v.s. classification accuracy for various aggregation methods. ~\Cref{fig:lv_bit_acc} shows ``concat'' for images are an outlier - lower bitrate with low accuracy. We will show in~\Cref{sec:norm} that ``concat'' disqualifies BB-ANS and Bit-Swap from being a \textit{normal} compressor. Please note the special case of ``concat'' doesn't apply to other compressors like gzip who treat images as bytes at the first place.

\begin{wrapfigure}{r}{0.5\textwidth}
  \begin{center}
    \includegraphics[width=0.48\textwidth]{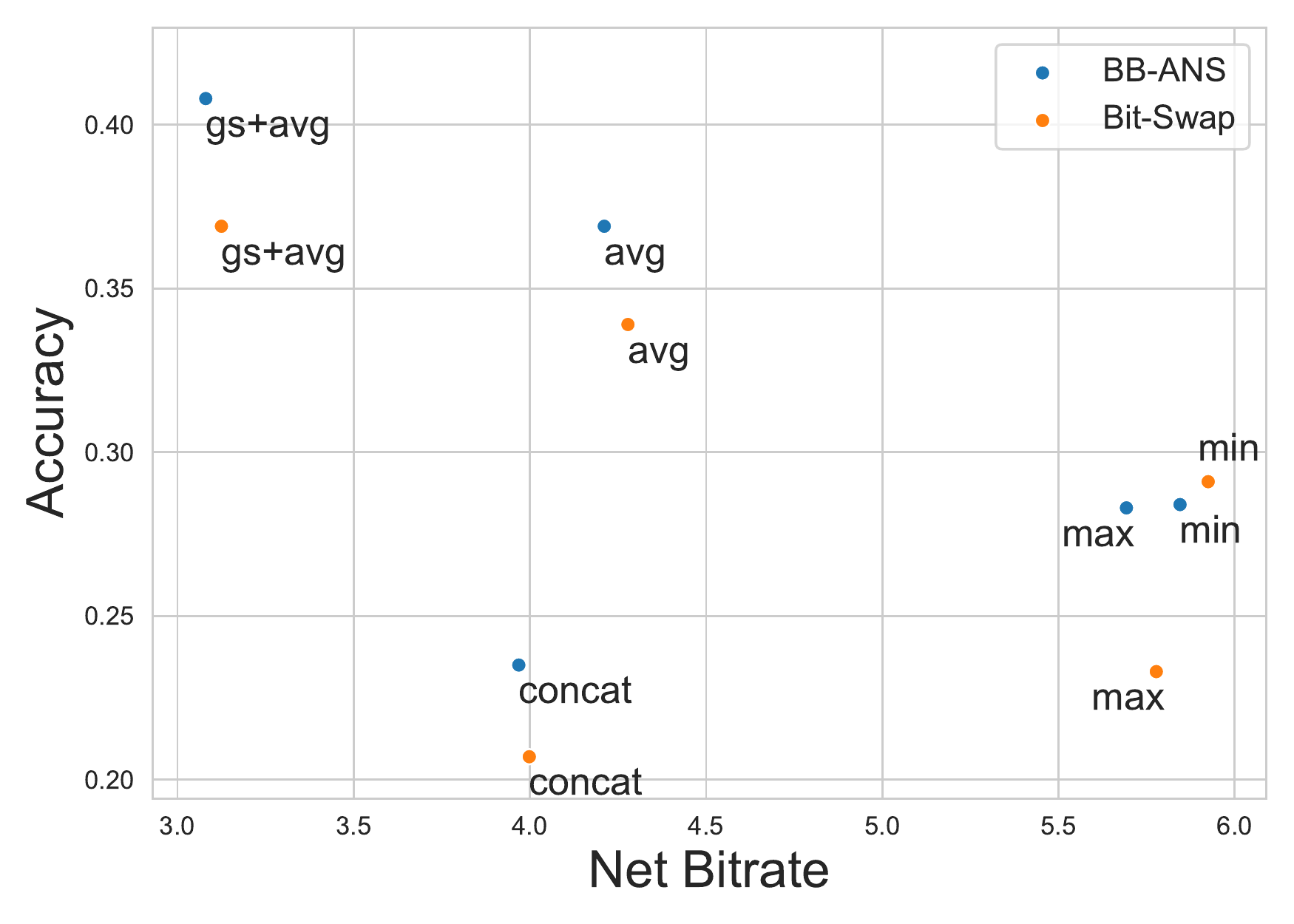}
  \end{center}
    \caption{Bitrate versus Accuracy on CIFAR-10 with various aggregation methods.}
    \label{fig:lv_bit_acc}
\end{wrapfigure}

Bit-Swap~\cite{kingma2019bit} achieves the new state of the art bitrate but why it doesn't suppress BB-ANS for classification? One of the reasons that Bit-Swap achieves better compression rate than BB-ANS is because Bit-Swap requires less initial bits, details shown in~\Cref{appx:init}. However, initial bits can only be amortized when there are multiple datapoints need to be compressed. But in our application, there are at most two datapoints need to be compressed sequentially. This leads us to choose the net bitrate, which excludes the length of the initial bits, and makes Bit-Swap less advantegeous. Theoretically, net bitrate should be the same for both methods. However, empirically we can see in~\Cref{fig:lv_bit_acc}, Bit-Swap uses slightly more net bits than BB-ANS.

Although we cannot draw a definite conclusion that the more a compressor can compress, the more accurate of classification NPC can obtain with the compressor. We can see that given a compressor, there is a correlation between bitrate and accuracy even with different aggregation methods except for aggregation methods that alter a compressor to be not \textit{normal} (e.g., ``concat').  

\textbf{Other Alternatives: }In this work, we don't go thoroughly through all the state of the art compressors, but only focus on the VAE-based lossless compressors. For this specific category, there are already numerous important factors: the choice of architectures, the choice of the number of latent variables, the choice of hierarchy topology (e.g., asymmetrical tree structure or symmetrical one) and the choice of discretization methods. Beyond this line of compressors, deep learning based compressors discussed in~\Cref{sec:npcdgm} can also be used under our framework. 

Beyond compressors, aggregation methods also cause diverging differences in the final classification accuracy. We covers a few basic ones but there are other non-training-required aggregation methods like \textit{linear blend operator} and even completely different aggregation strategies (e.g., using conditional VAE). On colored real-world images like CIFAR-10, image manipulation can be another easy and effective way to improve the accuracy. 
In effect,  ``greyscale'' can be viewed as ``lossy compression'' and this opens up the question of how lossy compressors perform under NPC framework.

\begin{table}[t]
    \centering
    \resizebox{.4\textwidth}{!}{%
    \begin{tabular}{|c|c|c|c|}
        \hline
        \# latent variables $z$ & 1 & 2 & 8 \\
        \hline
        Bit-Swap & 0.226 & 0.339 & 0.348 \\
        BB-ANS & 0.226 & 0.369 & 0.356 \\
        \hline
        \end{tabular}
        }
        \quad
        \resizebox{.51\textwidth}{!}{%
        \begin{tabular}{|c|c|c|c|c|c|}
        \hline
        aggregation & gs+avg & avg & min & max & concat \\
        \hline
        Bit-Swap & 0.369 & 0.339 & 0.291 & 0.233 & 0.207 \\
        BB-ANS & 0.408 & 0.369 & 0.284 & 0.283 & 0.235 \\
        \hline
    \end{tabular}
    }
    \vspace{1em}
    \caption{Effects of number of latent variables and aggregation methods.}
    \label{tab:z}
\end{table}


\section{Normal Compressor}
\label{sec:norm}


\begin{definition}[\textbf{Normal Compressor}]
A compressor is \textit{normal} if it satisfies, up to an additive $O(\log n)$ term, where $n$ means the maximal binary length of an element of $\Omega$:
\begin{enumerate}
    \item Idempotency: $C(xx)=C(x)$ and $C(\epsilon)=0$ where $\epsilon$ is the empty string
    \item Symmetry: $C(xy)=C(yx)$
    \item Monotonicity: $C(xy)\geq C(x)$
    \item Distributivity: $C(xy)+C(z)\leq C(xz)+C(yz)$
\end{enumerate}
\end{definition}

\begin{definition}[\textbf{Metric}]
\label{def:metric}
A distance function $D: \Omega\times\Omega \rightarrow \mathbb{R}^+$ is a \textit{metric} if it satisfies the following 3 criteria for any $x,y,z\in \Omega$, where $\Omega$ is a non-empty set, $\mathbb{R}^+$ represents the set of non-negative real number:

\begin{enumerate}
    \item Identity: $D(x,y) = 0$ iff $x=y$

    \item Symmetry: $D(x,y) = D(y,x)$
    
    \item Triangle Inequality: $D(x,y) \leq D(x,z)+D(z,y)$
\end{enumerate}
\end{definition}

\begin{definition}[\textbf{Admissible Distance}]
\label{def:aid}
A function $D:\Omega\times\Omega\rightarrow \mathbb{R}^+$ is an admissible distance if for every pair of objects $x,y\in\Omega$, the distance $D(x,y)$ is computable, symmetric and satisfies the density condition $\sum_y 2^{-D(x,y)}\leq 1$.
\end{definition}

\citet{cilibrasi2005clustering} formally prove that if the compressor is \textit{normal}, NCD is a normalized \textit{admissible} distance satisfying the metric inequalities, which is shown in~\Cref{def:metric}.~\citet{cebrian2005common} systematically evaluate how far real world compressors like gzip, bz2, PPMZ can satisfy the idempotency axiom. Here we empirically evalute all 4 axioms on MNIST with BB-ANS compressor. We randomly take 100 samples and plot $C(\cdot)$ on LHS as x-axis, $C(\cdot)$ on RHS as y-axis. For simplicity, we use one latent variable, under which circumstance BB-ANS equals to Bit-Swap. As shown in~\Cref{fig:norm_concat}, BB-ANS satisfy monotonicity, symmetry and distributivity. However, it fails on identity axiom, with $C(xx)\approx 1.988 C(x)$. Unlike gzip, bz2, or lzma, BB-ANS doesn't treat the concatenation of two images as a sequence of bytes but images, similar to PNG and WebP, making them hard to satisfy identity axiom. 
By default, the aggregation method refers to ``concatenation''. But simple concatenation does not perform well. Practically, BB-ANS fails the identity test when using ``concatenation''. Without changing the compressor, is it still possible to satisfy the above conditions so that NCD can be used as a normalized admissible distance metric?
We can change the aggregation method. 
We investigate whether BB-ANS with average can satisfy the above conditions. Obviously, identity and symmetry axioms can hold, as $C(\text{avg}(x,x))=C(\frac{x+x}{2})=C(x)$. We empirically evaluate monotonicity and distributivity, and find that they are both satisfied with ``average''. \Cref{fig:norm_avg} illustrates that $C(\text{avg}(x,y))\geq C(x)$ and $C(\text{avg}(x,y))+C(z)\leq C(\text{avg}(x,z))+C(\text{avg}(y,z))$ always hold. 
Given the compressor is normal under ``average'' function, we now prove NCD is an admissible distance metric.
\begin{definition}
\label{def:plus}
Let $D$ be an admissible distance. $D^+(x)$ is defined as $D^+(x)=\max\{D(x,z): C(z)\leq C(x)\}$, and $D^+(x,y)$ is defined as $D^+(x,y)=\max\{D^+(x), D^+(y)\}$
\end{definition}

\begin{lemma}
\label{lemma:ec}
If $C$ is a normal compressor, then $E_c(x,y)+O(1)$ is an admissible distance, where $E_c(x,y)=C(xy)-\min\{C(x)-C(y)\}$ is the compression distance.  
\end{lemma}

\begin{lemma}
\label{lemma:plus}
If $C$ is a normal compressor, then $E_c^+(x,y)=\max\{C(x), C(y)\}$
\end{lemma}

\begin{theorem}
If the compressor is normal, then the NCD is a normalized admissible distance satifying the metric (in)equalities.
\end{theorem}
\begin{proof}
\Cref{lemma:ec} and~\Cref{lemma:plus} show that NCD is a normalized admissible distance. We now show how NCD satisfies the metric (in)equalities.
\begin{enumerate}
    \item For identity axiom, 
    \begin{align}
        \text{NCD}(x,x)=\frac{C(\text{avg}(x,x))-C(x)}{C(x)}=0.
    \end{align}
    
    \item For symmetry axiom, 
    \begin{align}
        \text{NCD}(x,y) & =\frac{C(\text{avg}(x,y))-\min\{C(x),C(y)\}}{\max\{C(x),C(y)\}} = \text{NCD}(y,x).
    \end{align}
    \item For triangle inequality, without loss of generality, we assume $C(x)\leq C(y)\leq C(z)$. As $\text{NCD}$ is symmetrical, there are three triangle inequalities that can be expressed by $\text{NCD}(x,y), \text{NCD}(y,z), \text{NCD}(x,z)$. For simplicity, we prove one of them, $\text{NCD}(x,y)\leq \text{NCD}(x,z)+\text{NCD}(z,y)$ as the procedure for the other two is similar.
    Since BB-ANS is a \textit{normal} compressor under ``avg'', we have distributivity: $C(\text{avg}(x,y))+C(z)\leq C(\text{avg}(x,z))+C(\text{avg}(z,y))$. Subtracting $C(x)$ from both sides and rearraging results in $C(\text{avg}(x,y))-C(x)\leq C(\text{avg}(x,z))-C(x)+C(\text{avg}(z,y))-C(z)$. Dividing by $C(y)$ on both sides, we have \\
    \begin{align}
    & \frac{C(\text{avg}(x,y))-C(x)}{C(y)}  \leq \frac{C(\text{avg}(x,z))-C(x)+C(\text{avg}(z,y))-C(z)}{C(y)}.
    \end{align}
    We know LHS$\leq 1$ and RHS can be $\leq 1$ or $> 1$.
    \begin{enumerate}
        \item RHS$\leq 1$: Let $C(z)=C(y)+\Delta$; adding $\Delta$ to both the numerator and denominator of RHS increases RHS and makes it closer to 1.
        \begin{align}
        \frac{C(\text{avg}(x,y))-C(x)}{C(y)} &\leq \frac{C(\text{avg}(x,z))-C(x)}{C(y)+\Delta}  + \frac{C(\text{avg}(z,y))-C(z)+\Delta}{C(y)+\Delta}\\
        & = \frac{C(\text{avg}(x,z))-C(x)}{C(z)} + \frac{C(\text{avg}(z,y))-C(y)}{C(z)}.
        \end{align}
    
        \item RHS$> 1$: The procedure is similar to the case when RHS$\leq 1$. The difference is that adding $\Delta$ on both numerator and denominator makes RHS decrease instead of increase. But RHS cannot decrease less than 1. Thus, we still have
        \begin{align}
        & \frac{C(\text{avg}(x,y))-C(x)}{C(y)}  \leq \frac{C(\text{avg}(x,z))-C(x)}{C(z)} + \frac{C(\text{avg}(z,y))-C(y)}{C(z)}.
        \end{align}
    \end{enumerate}

\end{enumerate}
\end{proof}

\begin{figure}[h]
    \centering
    \includegraphics[width=.8\linewidth]{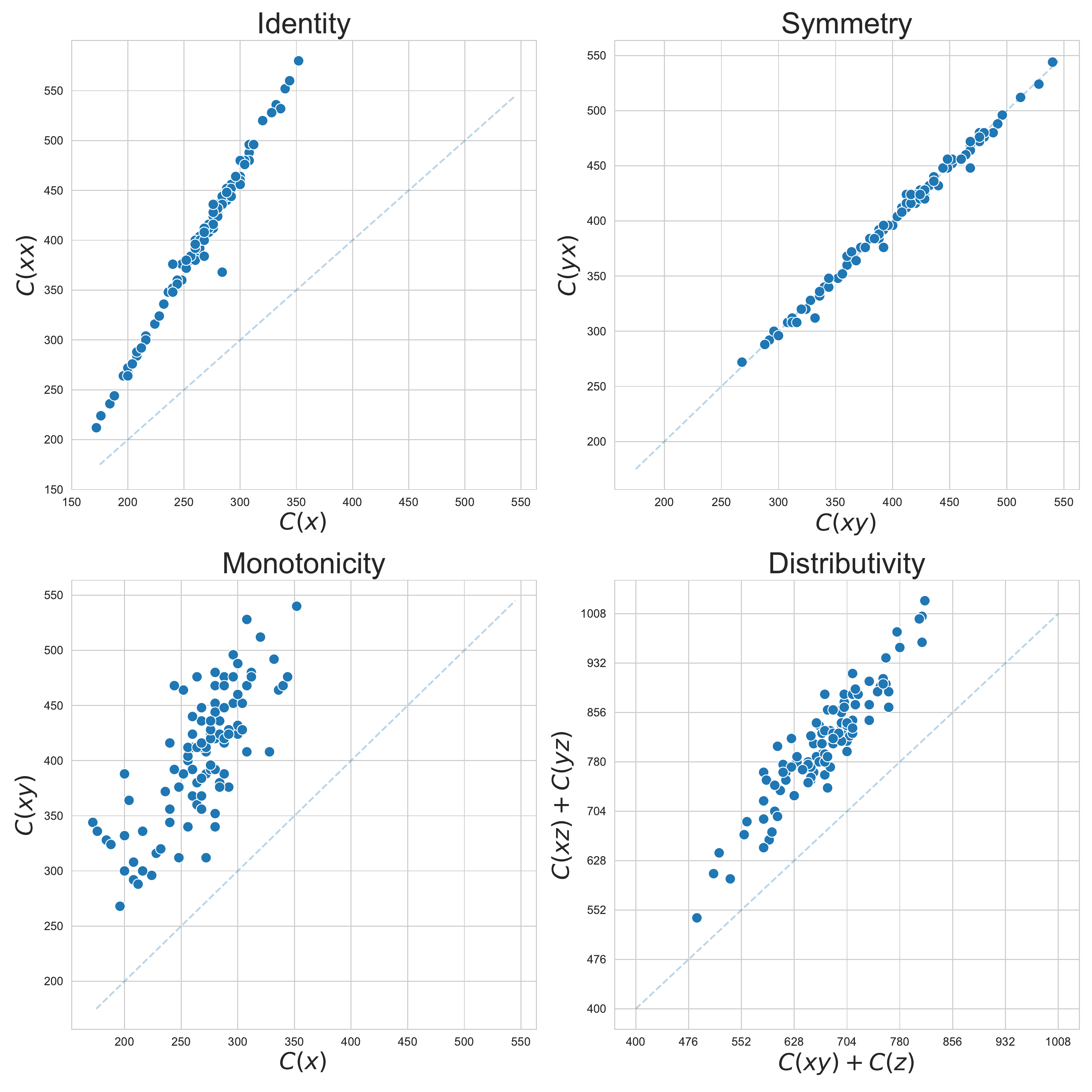}
    \caption{BB-ANS Normal Compressor Test for ``concatenation''}
    \label{fig:norm_concat}
\end{figure}

\begin{figure}[h]
    \centering
    \includegraphics[width=.8\linewidth]{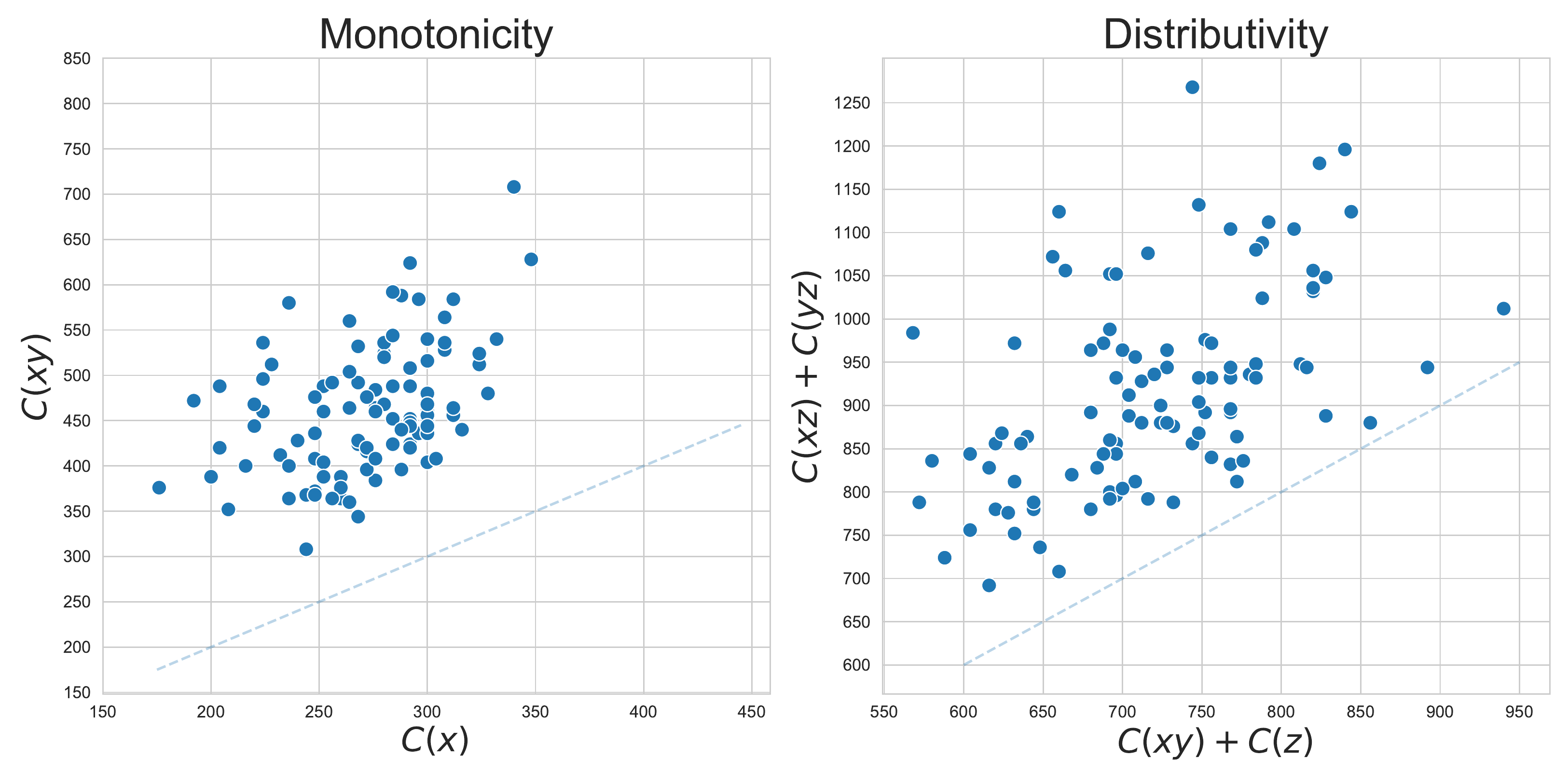}
    \caption{BB-ANS Normal Compressor Test for ``average''}
    \label{fig:norm_avg}
\end{figure}

\section{Training Details}
\label{appx:train}
For both CNN and VGG, we tune hyperparameters on a validation set to get the optimized performance. The reason that we use VGG11 instead of VGG16 or VGG19 is because VGG11 performs the best in low data regime in our experiments. For CNN, we first normalize MNIST, FashionMNIST and CIFAR-10. We use $\text{batch size}=4$, $\text{epoch number}=14$, Adadelta~\cite{zeiler2012adadelta} as the optimizer with $\text{learning rate}=1$, decaying learning rate by $\gamma=0.7$ every step for MNIST and FashionMNIST. Some of hyperparameters are from \href{https://github.com/pytorch/examples/tree/master/mnist}{PyTorch's official MNIST example}. We use the same hyperparameters except we increase the number of epoch to 20. For VGG11 on MNIST and FashionMNIST, we use $\text{epoch number}=20$ when given 50 samples per class, and use $\text{epoch number}=40$ when given less. We use Adam~\cite{kingma2015adam} with $\text{learning rate}=0.0001$ as the optimizer, and $\text{batch size}=4$. For CIFAR-10, we use $\text{learning rate}=0.00001$, $\text{epoch number}=80$.

For MeanTeacher and VAT, we follow~\citet{zhang2021flexmatch}'s \href{https://github.com/TorchSSL/TorchSSL}{implementation} and hyperparameters setting. We use WideResNet~\cite{zagoruyko2016wide} with $\text{depth=28}$ and $\text{widen factor}=2$ as the architecture for all three datasets and iterate 60,000 for MNIST and FashionMNIST, and iterate 200,000 for CIFAR-10. SGD with momentum is used for all three datasets, with $\text{learning rate}=0.03$, $\text{momentum}=0.9$.

For latent variable models used in this paper, we follow the training procedure and hyperparameters in~\citet{kingma2019bit} --- four `Processing' Residual blocks at the beginning of the inference model and the end of the generative model; eight `Ordinary' Residual block in total for all latent layers in both inference model and generative model. The Dropout~\cite{srivastava2014dropout} rate is 0.2 for MNIST and FashionMNIST, 0.3 for CIFAR-10. The learning rate for all datasets is $0.002$ with Adam optimizer. Dimension of latent variables for MNIST and FashionMNIST is $1\times 16\times 16$ while dimension of latent variables for CIFAR-10 is $8\times 16\times 16$. During $k$NN, we use $k=2$ for MNIST and FashionMNIST and $k=3$ for CIFAR-10.

We use one NVIDIA Tesla P40 GPU for training and compression. For pairwise computation in 50-shot setting, it takes roughly ten hours to calculate distance matrix on MNIST and FashionMNIST; it takes about thirty hours for CIFAR10. But once the distance matrix is calculated, evaluation on 5-shot or 10-shot just takes seconds. For both CNN and VGG, it takes about half an hour to train on 50-shot for one experiment. For VAT and MT, we need to re-train for every shot setting. For MNIST and FashionMNIST it takes about three hours to run one experiment in a single shot setting and for CIFAR10 it takes about twelve hours. The time that VAT and MT take positively relate to the number of iterations, which makes them slower than our method in the 5 and 10-shot but faster in the 50-shot setting.

\section{Details of ANS}
\label{appx:ans}
We briefly introduce ANS and show the proof of the optimal code length obtained from ANS.
The essence of ANS is to encode one or more data points into a single natural number, called state $s\in \mathbb{N}$. Depending on different vocabularies and manipulations, there are different variations of ANS (details can be seen in~\citet{duda2015use}. We introduce one of them - rANS (range ANS), which is the variant we use in this paper. The notation we use here is unconventional in order to be consistent with the main part of the paper.

Let’s notate our state at timestamp $t$ as $s_t\in \mathbb{N}$, and notate our symbol/message at $t$ as $x_t$, $x_t\in V$, where $V=\{0,1\}$ is the vocabulary set.
We have two simple methods to encode a binary sequence into a natural number bit by bit --- $s_t = 2s_{t-1}+x_t$ or $s_t = s_{t-1}+2^mx_t$. The former means appending information to the least significant position while the latter is adding information to the most significant position. It's obvious that encoding new symbol into the most significant position requires remembering $m$ while encoding in the least significant position only need previous state $s_{t-1}$ and new information $x_t$. 
It's also easy to decode: depending on whether the current state $s_t$ is even or odd, we not only know if the last encoded symbol $x_t$ is $0$ or $1$, but we can also decode the state following $s_{t-1} = \frac{s_t}{2}$ or $s_{t-1} = \frac{s_t-1}{2}$. 

The above example illustrates encoding and decoding methods when there are two elements in the vocabulary with uniform distribution $p(0)=p(1)=\frac{1}{2}$. In this case, it's optimal to scale up $s_t$ to two for both $0$ or $1$ as we essentially only spend 1 bit per encoded symbol. However, when the probability is not uniformly distributed, the entropy is smaller, and scaling up by 2 for both symbols will not be optimal anymore.
rANS generalizes the process to any discrete probability distribution and any size of vocabulary.

Intuitively, scaling up by a smaller factor for a more probable symbol and scaling up by a larger factor for a less probable symbol will provide us with a more efficient representation.
Concretely, we have a sequence of messages $\xv = (x_1,x_2,x_3,...,x_n)$, and a vocabulary $V=\{v_1, v_2, v_3, ..., v_k\}$, with size $k$, $x_i\in V$. We also have probability mass distribution of $V$: $P=\{p_{v_1}, p_{v_2}, p_{v_3}, ..., p_{v_k}\}$. Correspondingly, let's define frequency counts $F=\{f_{v_1}, f_{v_2}, f_{v_3}, ..., f_{v_k}\}$, $f_{v_i} = p_{v_i} \times M$ where $M=\sum_{i=1}^k f_{v_i}$. $M$ can be viewed as a multiplier, demonstrating the precision of ANS, which is a predefined variable in the implementation. We can also get cumulative frequency counts from $F$ as follows $B=\{b_{v_1}, b_{v_2}, b_{v_3}, ..., b_{v_k}\}$ where $b_{v_i}=\sum_{j=1}^{i-1}f_{v_j}$. Now we get everything we need to define the encoding function $G$:

\begin{equation}
\label{eq:ans_enc}
\begin{split}
    & s_t = G(s_{t-1}, x_t), \\
    & G(s_{t-1}, x_t) = \floor*{\frac{s_{t-1}}{f_{x_t}}}\times M + b_{x_t} + s_{t-1}~\text{mod}~f_{x_t}.
\end{split}
\end{equation}

The procedure can be interpreted as follows: We have various $M$-sized blocks partitioning natural number $\mathbb{N}$. Encoding can be viewed as finding the exact location of the natural number that represents the state, by first finding the corresponding block ($\floor*{\frac{s_{t-1}}{f_{x_t}}}\times M$) followed by finding the sub-range representing that symbol within $M$ ($b_{x_t}$) and finding the exact location within that sub-range ($s_{t-1}~\text{mod}~f_{x_t}$). 
The decoding function $H(s_t)$ is the reverse of the encoding:
\begin{equation}
\begin{split}
    & s_{t-1}, x_t = H(s_t), \\
    & x_t = \text{argmax}~\{b_{x_t} < (s_{t}~\text{mod}~M) \}, \\
    & s_{t-1} = f_{x_t}\floor*{\frac{s_t}{M}}+s_{t}~\text{mod}~M-b_{x_t}.
\end{split}
\end{equation}
We first find the precise location within the sub-range using inverse function of cumulative counts ($\text{argmax}~\{b_{x_t} < (s_{t}~\text{mod}~M) \}$). With $x_t$ we can reverse steps in~\Cref{eq:ans_enc} to get the previous state. 
As we can see ANS decodes in the reverse order of encoding (i.e., last in first out), which makes it compatible with bits-back argument. 

From encoding function, we know that:
\begin{equation}
    \frac{s_t}{s_{t-1}}\approx \frac{M}{f_{x_t}} = \frac{1}{p_{x_t}}.
\end{equation}
Encoding a sequence of symbols $\xv$ results in:
\begin{equation}
    s_n\approx \frac{s_0}{p_{x_1}p_{x_2}...p_{x_n}}.
\end{equation}
Thus, the total coding length is:
\begin{equation}
    \log s_n \approx \log s_0 +  \sum_{i=1}^n \log \frac{1}{p_{x_i}},
\end{equation}
where $s_0$ refers to the initial state. Dividing by $n$ we will get the average coding length that approximates the entropy of the data.

\section{Discretization}
\label{appx:disc}
ANS is defined for symbols in a finite alphabet; bits-back coding works for discrete latent variables. However, continuous latent variables are proven to be powerful in many latent variable models. In order to use those latent variable models for lossless compression, discretizing continuous variables into discrete ones is a necessary step.
\citet{townsend2018practical} derives, based on~\citet{mackay2003information}, that using bits-back coding continuous latent variables can be discretized to arbitrary precision without affecting the compression rate. Suppose a probability density function $p$ is approximated using a number of ``buckets'' of equal width $\sigma\zv$. For any given bucket $j$, we can know its probability mass $p(\zv^{(j)})\sigma\zv$ where $\zv^{(j)}$ is some point in the bucket $j$. Let's notate the discrete distribution as $P$ and $Q$ for both prior and posterior distribution. Then for any given bucket $j$, $P(j)\approx p(\zv^{(j)})\sigma \zv$. The expected message length with a discretized latent variable is:
\begin{equation}
    -\mathbb{E}_{Q(j|\xv)}\log \frac{p(\xv|\zv^{(j)})p(\zv^{(j)})\sigma\zv}{q(\zv^{(j)}|\xv)\sigma\zv}.
    \label{eq:dis}
\end{equation}
The width of buckets $\sigma\zv$ is cancelled. Therefore, as long as the bins for inference models match the generative models, continuous latent variables can be discretized up to an arbitrary precision.

In this paper we only consider the basic discretization techniques like dividing continuous distribution into bins with \textit{equal width} or \textit{equal mass}. We discretize the prior (top layer) with equal mass and all subsequent latent layers with equal width. As~\Cref{eq:dis} shows, ideally we want the discretization aligns between inference models and generative models. However, discretization of $\zv_i\sim p_\theta(\zv_i|\zv_{i+1})$ relying on $\zv_i\sim q_\phi(\zv_i|\zv_{i-1})$ is not possible without sampling. In the compression stage, when decoding $\zv_i$, $\zv_{i+1}$ is not available and so is $p_\theta(\zv_i|\zv_{i+1})$. In the decompression stage, similarly, $q_\phi(\zv_i|\zv_{i-1})$ is not available for $p_\theta(\zv_i|\zv_{i+1})$ to match with. Therefore, we need to sample from training dataset beforehand to get unbiased estimates of the statistics of the marginal distribution~\cite{kingma2019bit}. This process only needs to be done once and can be saved for the future use.

\section{Initial Bits of BB-ANS and Bit-Swap}
\label{appx:init}
The main difference between BB-ANS and Bit-Swap is that BB-ANS requires the sender \textit{Alice} to decode $\zv_{i+1}$ with $q_\phi(\zv_{i+1}|\zv_i)$ for $i$ from 1 to $L-1$ first, and then encode $\zv_i$ with $p_\theta(\zv_i|\zv_{i+1})$ for $i$ from 1 to $L-1$. While Bit-Swap interleaves this encoding and decoding procedure and applies recursively for latent variables, as illustrated in~\Cref{fig:bbans-bitswap}. The advantage of Bit-Swap's procedure is that, after decoding $\zv_1$, the bits encoded from $x$ can be taken advantage in decoding $\zv_2$; then bits encoded from $\zv_1$ can be used for decoding $\zv_3$. As a result, the initial bits required for Bit-Swap is much less than BB-ANS. Concretely, for BB-ANS, the minimum initial bits required:
\begin{equation}
    -\log q_\phi(\zv_1|\xv)-\sum_{i=1}^{L-1}\log q_\phi(\zv_{i+1}|\zv_i).
\end{equation}
For Bit-Swap, the minimum initial bits requires:
\begin{equation}
\begin{split}
    & -\max~(0, \log q_\phi(\zv_1|\xv))+\sum_{i=1}^{L-1}\max~\left(0, \log\frac{p_\theta(\zv_{i-1}|\zv_i)}{q_\phi(\zv_{i+1}|\zv_i)}\right) \\
    & \leq -\log q_\phi(\zv_1|\xv)-\sum_{i=1}^{L-1}\log q_\phi(\zv_{i+1}|\zv_i).
\end{split}
\end{equation}
The initial bits Bit-Swap requires is less than BB-ANS, making Bit-Swap reaching the optimal compression rate.

\section{Hierarchical Latent Variable Models}
\label{appx:hier}
The hierarchical autoencoder in the paper uses deep latent gaussian models (DLGM)~\cite{rezende2014stochastic} following the sampling process based on Markov chains, whose marginal distributions are:
\begin{equation}
\begin{split}
    p_\theta(\xv) &= \int p_\theta(\xv|\zv_1)p_\theta(\zv_1)d\zv_1, \\
    p_\theta(\zv_1) &= \int p_\theta(\zv_1|\zv_2)p_\theta(\zv_2)d\zv_2, \\
    &... \\
    p_\theta(\zv_{L-1}) &= \int p_\theta(\zv_{L-1}|\zv_L)p_\theta(\zv_L)d\zv_L.
\end{split}
\end{equation}
Combining the above equations, the marginal distribution of $\xv$ is:
\begin{equation}
    p_\theta(\xv) = \int p_\theta(\xv|\zv_1)p_\theta(\zv_1|\zv_2)...p_\theta(\zv_{L-1})|\zv_L)p_\theta(\zv_L)d\zv_{1:L}.
\end{equation}
Accordingly, inference models $q_\phi(\zv_{i+1}|\zv_i)$ need to be defined for every latent layer. ELBO that includes multiple latent variables then becomes:
\begin{equation}
    \mathbb{E}_{q_\phi (\cdot|\xv)}[\log p_\theta(\xv, \zv_{1:L})-\log q_\phi (\zv_{1:L}|\xv)].
\end{equation}

In this paper, we use Logistic distribution ($\mu=0, \sigma=1$) as the prior $p(\zv_L)$, and use conditional Logistic distribution for both inference models $q_\phi(\zv_{i+1}|\zv_i), q_\phi(\zv_1|\xv)$ and generative models $p_\theta(\zv_i|\zv_{i+1})$. These distributions are modeled by neural networks, which is stacked by Residual blocks~\cite{he2016deep} as hidden layers. 
More architecture details can be referred to~\citet{kingma2019bit}, where they also discusses other possible topologies regarding to hierarchical latent variable models.

\section{Proof of Universality of Information Distance}
\label{proof:uni}
\textit{Information Distance} $E(x,y)$ refers to the length of the shortest binary program generated by universal prefix Turing machine, that with input $x$ computes $y$, and with input $y$ computes $x$. It's shown that $E(x,y)=\max\{K(x|y), K(y|x)\}$. We now prove~\Cref{theo:id}, based on~\Cref{lemma:uni}~\cite{bennett1998information}.
\begin{lemma}
\label{lemma:uni}
For every upper-semicomputable function $f(x,y)$, satisfying $\sum_{y}2^{-f(x,y)}\leq 1$, we have $K(y|x) < f(x,y)$.
\end{lemma}

To prove that $E(x,y)$ is a metric, we show it satisfies metric (in)equalities. We can infer the non-negativity and symmetry directly from the definition $E(x,y)=\max\{K(x|y), K(y|x)\}$. For triangle inequality, given $x,y,z$, without loss of generality, let $E(x,z)=K(z|x)$. By the self-limiting property, we have
\begin{equation}
\begin{split}
    E(x,z) & = K(z|x) < K(y,z|x) < K(y|x)+K(z|x,y) \\
    & < K(y|x)+K(z|y) \leq E(x,y) + E(y,z).
\end{split}
\end{equation}
To prove $E(x,y)$ is \textit{admissible}, we show it satisfies density requirement:
\begin{equation}
\sum\limits_{y:y\neq x}2^{-E(x,y)}\leq \sum\limits_{y:y\neq x}2^{-K(y|x)}\leq 1.
\end{equation}
The second inequality is due to Kraft's inequality for prefix codes.\\
To prove the minimality, as for every admissible distance metric $D(x,y)$, it satisfies $\sum\limits_{y:y\neq x}2^{D(x,y)}\leq 1$. According to \Cref{lemma:uni}, we have $K(y|x)<D(x,y)$ and $K(x|y)<D(y,x)$.

\section{Potential Negative Social Impact}
\label{appx:neg}
We propose a general framework for classification and clustering tasks with minimum assumption about the dataset. We don't foresee it has negative social impact itself. However, our framework uses generative models for their explicit density estimation and the development of generative models may be used by people with ulterior motives to generate fake data.

\end{document}